\documentclass{article}
\usepackage{graphicx}
\usepackage{latexsym}

\usepackage{url}
\usepackage{amssymb,amsmath}
 \newcommand{\Bem}[1]{}

\usepackage{graphicx}
\usepackage[utf8]{inputenc}
\usepackage[boxed]{algorithm2e}
\SetKw{KwBy}{by}

\usepackage{amstext}
\usepackage{amsthm}
\usepackage{latexsym}

\usepackage{url}
\usepackage{amssymb,amsmath}
\usepackage{ifpdf}
\ifpdf
\usepackage[colorlinks,linkcolor=blue,urlcolor=blue,citecolor=blue,
plainpages=false,pdfpagelabels,breaklinks]{hyperref}
\else
\usepackage[colorlinks,linkcolor=blue,urlcolor=blue,citecolor=blue,
plainpages=false,pdfpagelabels,linktocpage]{hyperref}
\fi
\usepackage{hyperref}

\newtheorem{twierdzenie}{Theorem}
\newtheorem{definicja}{Definition}

\begin{document}

\title{Wide Gaps and Clustering Axioms}

\author{Mieczys{\l}aw A. K{\l}opotek
\\Institute of Computer Science, \\Polish Academy of Sciences,\\ ul. Jana Kazimierza 5, 01-248 Warsaw, Poland\\ \url{klopotek@ipipan.waw.pl}, 
\url{http://www.ipipan.waw.pl} 
}

\allowdisplaybreaks

\maketitle

\begin{abstract} 
The widely applied $k$-means algorithm produces clusterings that violate our expectations with respect to  high/low similarity/density and is in conflict with Kleinberg’s axiomatic system  for distance based clustering algorithms that formalizes those expectations in a natural way. $k$-means violates in particular the consistency axiom. We hypothesise that this clash is due to the not explicated expectation that the data themselves should have the property of being {cluster}able in order to expect the algorithm clustering hem to fit a clustering axiomatic system.  
To demonstrate this, we introduce two new clusterability properties,  variational $k$-separability and residual $k$-separability and show that then the Kleinberg's consistency axiom holds for $k$-means operating in the Euclidean or non-Euclidean space. 
Furthermore, we propose extensions of $k$-means algorithm that fit approximately the Kleinberg’s richness axiom that does not hold for $k$-means. 
In this way, we reconcile $k$-means with Kleinberg's axiomatic framework in Euclidean and non-Euclidean settings. Besides contribution to the theory of axiomatic frameworks of clustering and for clusterability theory, practical contribution is the possibility to construct {data}sets for testing purposes of algorithms optimizing $k$-means cost function. This includes a method of construction of {cluster}able data with known in advance global optimum.

\Bem{
}
\end{abstract}
%

\section{Introduction}

Clustering is a domain of machine learning with quite vague foundations. The concept of a cluster or a clustering is poorly defined. It is associated with high within-cluster similarity and low between-cluster similarity, with high density areas separated with low density areas, with optimizing some cost function, with matching manually assigned labels, with various internal and external clustering scores etc
(see e.g. {Madhulatha \cite{Madhulatha:2012}).
Also various axiomatic systems have been designed defining clustering related properties,like that of Kleinberg \cite{Kleinberg:2002}. 
The conceptual problem with all these definitions is that the widely applied $k$-means algorithm in its base form and derivatives does not care about high/low similarity/density etc. and violates two of three Kleinberg's axioms for distance based clustering algorithms, while checking if the optimum of its cost function is reached would require enumeration of all possible clusterings, and hence is prohibitive in practice. 

The special attention that we pay here to Kleinberg's axiomatic system is due to the fact that two of his axioms induce a method for generating new test datasets from existing ones without the need of manual labelling of the new sets. This is important because 
development and implementation of new algorithms in the area of machine learning, especially clustering, comparative studies of such algorithms as well as testing according to software engineering principles require availability of labeled data sets. While standard benchmarks are made available, a broader range of such data sets is necessary in order to avoid the problem of overfitting. 
In this context, theoretical works on axiomatization of clustering algorithms, especially axioms on clustering preserving transformations like that of Klein{}berg \cite{Kleinberg:2002} are quite a cheap way to produce labeled data sets from existing ones, given that the respective algorithm to be tested fits the axiomatic framework.  

However, $k$-means algorithmic family does not fit the ``natural'' Kleinberg's axiomatic framework (the richness and consistency axioms are violated). So, what is wrong about this framework? It may be hypothesised that data that have really a clustering structure (is {cluster}able) will behave according to Kleinberg's intuition, while at the same time we cannot expect such a behaviour when the data does not have the clusterability property.  
In this paper we demonstrate that this hypothesis is accurate with respect to the $k$-means. 

We recall earlier work in Section \ref{sec:prevWork}.
Then   in Section \ref{sec:varsep} we demonstrate that 
if the data has the clustrerability property that we call variational $k$-separability then the Kleinberg's consistency axiom holds for $k$-means.  
Furthermore, it is possible to construct a $k$-range-means algorithm, that generalizes $k$-means by automatic selection of $k$, for which all three Kleinberg's axioms hold for data with variational $k$-range-separability when adding the restriction to consistency axiom that data concentrations within a cluster are not created. 

The deficiency of the proposal in  Section \ref{sec:varsep}  is that it is  applicable to Euclidean space only, while so-called kernel-$k$-means operates de-facto in non-Euclidean space(see e.g. \cite{Girolami:2002,STWMAK:2018:clustering}).  To overcome this restriction, we propose in Section \ref{sec:residualsep} the clusterability concepts of residual $k$-separability and residual $k$-range-separability which imply the restriction of consistency to the realistic case of finite measurement resolution.  Section 
\ref{sec:non-eucliden}
 explains how these concepts apply to non-Euclidean spaces. 

Section   \ref{sec:conclusions} summarizes the results and outlines further research directions. 

The main contributions of this paper are proposals of clusterability criteria that reconcile $k$-means with Kleinberg's axiomatic framework in Euclidean and non-Euclidean settings and also a method of construction of {cluster}able data with known in advance global optimum. 
Besides contribution to the theory of axiomatic frameworks of clustering and for clusterability theory, practical contribution is the possibility to construct {data}sets for testing purposes of algorithms optimizing $k$-means cost function. We propose also a generalization of $k$-means algorithm that can self-adjust $k$  when the data is {cluster}able in the mentioned way.

\section{Previous Work} \label{sec:prevWork}
We will refer in this paper to the widely used $k$-means  algorithm ($k$-means++ version, by Arthur and   Vassilvitskii \cite{CLU:AV07}), which belongs to the so-called \emph{$k$-clustering  algorithms} that is, for a dataset $S$ they return a partition $\Gamma$ of $S$ into  $k$  non-empty groups ($|\Gamma|=k$), where $k$ is a user-defined parameter. 
$k$-means algorithm was designed  to operate primarily in the Euclidean space, that is we assume an embedding $\mathcal{E}: S \rightarrow \mathbb{R}^d$ into a $d$-dimensional Euclidean space.    
$k$-means seeks to find a partition $\Gamma$ of $S$ that minimizes the cost  (or quality) function  
 $$Q(\Gamma)=\sum_{C \in \Gamma} \sum_{e \in C} ||\mathcal{E}(e)-\boldsymbol\mu(C)||^2$$
 where $\boldsymbol\mu(C)=\frac1{|C|} \sum_{e \in C} \mathcal{E}(e) $
 which may be reformulated as 
\begin{equation} \label{eq:Q::kmeans}
Q(\Gamma)
=\sum_{C \in \Gamma} \frac{1}{2|C|} \sum_{i \in C} 
\sum_{l \in C} \|\mathcal{E}(i)  - \mathcal{E}(l)\|^2  
\end{equation} 
Kleinberg \cite{Kleinberg:2002} proposed three seemingly obvious clustering axioms for distance based clustering algorithms: 
richness, scale-invariance and consistency. 
Hereby an algorithm is a function $f(S,d)=\Gamma$ producing a partition $\Gamma$ of $S$ given the (pseudo)distance function $d:S\times S\rightarrow \mathbb{R}$ such that $d(i,i)=0$, $d(i,l)=d(l,i)\ge 0$,  where $d(i,j)=0$ iff $i=l$.  
The \emph{consistency axiom} states   that if  $f(S,d)=\Gamma$ and $d'$ is another distance function such that  $d'(i,l)\ge d(i,l)$ iff $i,l$ are from different clusters of $\Gamma$ and  $d'(i,l)\le d(i,l)$  iff $i,l$ are from the same cluster, then $f(S,d')=\Gamma$.
\emph{Scale invariance} means if the distance $d'$ has the property that for an   $\alpha\in\mathbb{R}^+$, $d'(i,l)=\alpha d(i,l)$, then $f(S,d)=f(S,d')$. \emph{Richness} means that for any $S$ and for any its partition $\Gamma$ there exists a (pseudo) distance function $d$ such that $f(S,d)=\Gamma$.

The three axioms proved to be  contradictory that is no clustering algorithm can fulfil all three requirements at once (see proof in \cite{Kleinberg:2002}). 
Furthermore,  $k$-means \emph{is not a clustering algorithm} as it fails on richness and consistency axioms. 
$k$-means is not rich as it returns only such $\Gamma$ that $k=|\Gamma|$. 
So by weakening Kleinerg's axiom of richness to \emph{$k$-richness} (richness restricted to partitions $\Gamma$ with $|\Gamma|=k$), we can get rid of violation of richness. Violation of consistency axiom remains, however (see the proof in \cite{Kleinberg:2002}). 

It is  disastrous  for the domain of clustering algorithms if an axiomatic system consisting of ''natural axioms'' is self-contradictory. It means that the domain of clustering algorithms is a kind of fake science. Therefore numerous efforts have been made to cure such a situation by proposing different axiom sets or modifying Kleinberg's theory.  Kleinberg himself introduced the concept of partition $\Gamma'$ being a refinement of a partition $\Gamma$, 
if for every set $C' \in \Gamma'$, there is a
set $C \in  \Gamma$ such that $C' \subseteq C$.
He  defines Refinement-Consistency, a relaxation of Consistency, to require that
if distance d' is a consistency transformation of d, then f(S,d') should be a refinement of f(S,d) or vice versa.
Though  there is no clustering function that satisfies Scale-{In}variance, Richness, and Refinement-Consistency, but if one defines Near-Richness as Richness without  the partition in which each element is in a separate cluster, then  there exist clustering functions f that satisfy Scale-{In}variance and
Refinement-Consistency, and Near-Richness (e.g.  single-linkage with the distance-($\alpha\delta$) stopping condition, where
$\delta=  min_{i,j} d(i, j)$ and $\alpha\ge 1$.)  

The refinement consistency does not allow to generate labelled new data from existent labelled old data.   Furthermore, it does not repair the $k$-means classification as a non-clustering algorithm.   
To overcome Kleinberg's contradictions, Ben-David and Ackerman 
\cite{Ben-David:2009} proposed to axiomatize clustering quality function and not the clustering function itself. Regrettably, no requirements are imposed onto the clustering function itself. This means that labelled {data}sets cannot be derived automatically from existent ones. Van Laarhoven and Marchior
\cite{vanLaarhoven:2014} propose to go over to the realm of graphs and develops a set of axioms for graphs. The approach is not applicable to $k$-means. Ackerman et al. 
\cite{%
Ackerman:2010NIPS} and Meila  \cite{Meila:2005} proposed to use the ``axioms'' not as a requirement to be met by all algorithms, but rather as a way to classify clustering functions.  
Strazzeri et al. \cite{Strazzeri:2021} suggests to change the consistency axiom for graphs. 
Hopcroft and Kannan \cite{Hopcroft:2012}  propose to seek only clusters with special properties, in this case  to cluster the {data}sets into equal size clusters. 
Cohen et al. \cite{Cohen:2018} suggest to modify consistency axiom in that they require that Kleinberg's consistency holds only if the optimal number of clusters prior and after his $\Gamma$ transformation remains the same. 
Though they show that various algorithms, including $k$-means fit this new axiom, the problem is of course that you are usually unable to tell apriori the optimal number of clusters, hence usage of such an axiomatic set as a tool for test set generation is pointless.

We have also proposed several approaches to removing the contradictions in the Kleinberg's axiomatic system, see e.g. \cite{MAKRAK:2020:limcons,MAKSTWRAK:2020:motion,MAKRAK:2020:fixdimcons,RAKMAK:2019:probrich,MAK:2022:kmeanspreserving,RAK:MAK:2019:perfectball,Klopotek:2022continuous,ISMIS:2022:richness}. All of them were based on the enclosure of clustrers into balls and keeping gaps between balls large. These approaches were valid only for Euclidean spaces.   

The proposals in this paper are inspired by the research on so-called  clusterability.  As mentioned, \cite{Hopcroft:2012} made a suggestion that restricting oneself to special data structures can overcome Kleinberg's contradictions. That is one looks rather at clustering of data that fulfil some properties of clusterability. 
Though a number of attempts have been made to capture formally the intuition behind clusterability, none of these efforts seems to have been successful, as Ben-David exhibits in \cite{Ben-David:2015} in depth.  A   paper by Ackerman et al.  \cite{Ackerman:2016} partially eliminates some of these problems, but regrettably at the expense of non-intuitive  user-defined parameters. 
As Ben-David mentioned,   the research in the area  does not address popular algorithms  except for $\epsilon$-{Separated}ness clusterability criterion related to $k$-means proposed  by Ostrovsky et al. \cite{Ostrovsky:2013}.
We have made some efforts in this direction in \cite{MAK:2017:clusterability}. This paper also refers to clusterability while clustering via $k$-means. 

The issue of clustering axiomatisation is closely related to the problem of cluster preserving transformations in general. Such transformations are of vital importance because they may be used  to the problem of test{}bed creation for clustering algorithms. 

Roth et al. \cite{Roth:2003} investigated the issue of preservation of clustering when embedding non-euclidean data into the Euclidean space. 
They showed that clustering functions, that remain invariant under additive shifts of the pairwise proximities, can be reformulated as clustering problems in Euclidean spaces. 

A similar problem was addressed in 
\cite{RAKMAKSTW:2020:trick} whereby the issue of interpretation of results of kernel $k$-means to non-euclidean data was discussed. A cluster-preserving transformation for this specific problem was proposed via increasing of all distances. 

Parameswaran and   Blough
\cite{Parameswaran:2005}
considered the issue of  cluster preserving transformations from the point of view of privacy preserving. They designed a 
Nearest Neighbor Data Substitution (NeNDS), a new  data obfuscation technique with  strong privacy-preserving properties while  maintaining data clusters.
Cluster preserving transformations with the property of 
privacy preserving focusing on the $k$-means algorithm are investigated by Ramírez and  Auñón \cite{Ramirez:2020privacy}.
Privacy preserving methods for various $k$-means variants boosted to large scale data are further elaborated by 
 Gao and   Zhan
\cite{Gao:2017}. 
Keller et al. \cite{Keller:2021} investigate such transformations for other types of clustering algorithms. 
A thorough survey of  
privacy-preserving clustering for big data
can be found in \cite{Zhao:2020} by Zhao et al.

Howland  and  Park
\cite{Howland:2007} 
proposed  models incorporating prior knowledge about the existing structure and developed for them  dimension reduction methods independent of the original  term-document matrix dimension.
Other, more common dimensionality reduction methods for clustering (including PCA and Laplacian embedding) are reviewed  by 
Ding \cite{Ding:2009}. 

Larsen et al. 
\cite{Larsen:2016heavy}
reformulate the heavy hitter problem of stream mining in terms of a clustering problem and elaborate algorithms fulfilling the requirement of ``cluster preserving clustering''. 

Zhang et al. 
\cite{Zhang:2019}
developed clustering structure preserving transformations for graph streaming data, when there is a need to sample the graph.

\section{Variation{al} Cluster Separation} \label{sec:varsep}

Let us introduce a couple of useful concepts. 
First of all recall the fact that Kleinberg's consistency axiom leads definitely outside of the domain of Euclidean space. Therefore, to work with $k$-means algorithm, we need a reformulation of the $k$-means cluster quality function.
Let us first recall Kleinberg's ``distance'' concept. 
\begin{definicja}
For a given discrete set of points $S$, 
the function $d: S \times S \rightarrow \mathbb{R}$ will be called  a \emph{pseudo-distance function}  iff  $d(x,x)=0$, $d(x,y) =d(y,x)$ and   $d(x,y) > 0$ for distinct $x,y$.
\end{definicja}

Following the spirit of kernel $k$-means as was exposed in \cite{RAKMAKSTW:2020:trick}, 
let us reformulate the $k$-means cluster quality function in terms of this pseudo-distance. 

Define the function $Q(\Gamma,d)$  as follows:
\begin{equation}\label{eq:kmeansPseudoDist}
Q(\Gamma,d)   
= \sum_{C \in \Gamma} \frac{1}{2|C|} \sum_{i \in C} 
\sum_{l \in C} d(i,l) ^2  
\end{equation}   
where $\Gamma$ is a clustering (split into disjoint non-empty subsets of cardinality at least 2) of a dataset $S$ into $k$ clusters, and $d$ is a pseudo-distance function defined over $S$. $Q(\Gamma,d)   $ generalizes $Q(\Gamma)$ from formula (\ref{eq:Q::kmeans}) in that it allows non-Euclidean distances. 

Let us introduce our concept of well-{separated}ness. 
\begin{definicja} \label{def:varsep}
Let us consider a set of clusters separated as follows: 
Let $\Gamma=\{C_1,\dots,C_k\}$ be a partition of the dataset $S$, $d$ be a pseudo-distance. 
Let   
\begin{equation}\label{eq:vardist}
d(i,l)>\sqrt{2}\sqrt{Q(\Gamma,d)} 
\end{equation}
for each $i,l$ 
such that $i$ belongs to a different cluster than $l$ under $\Gamma$.  
Then we say that the set $S$ with distance $d$ is 
\emph{ variation{ally} $k$-separable}
and that this  $\Gamma$ is 
\emph{ variation{al} $k$-separation} of $S$.
If furthermore, no cluster of $\Gamma$ has the property of variation{al} $k'$-separation for all $k'=2,\dots,K+1$ for some integer $K\ge 2$, then  
 $\Gamma$ is 
\emph{ variation{al} $k+K$-range-separation} of $S$. 
\end{definicja}

It is easily seen that in such a case
\begin{twierdzenie}\label{thm:varksepISoptimal}
If the pseudo-distance $d$ fulfills the condition (\ref{eq:vardist}) 
under the clustering 
$\Gamma$ of $S$, then $\Gamma$  is the optimal $k$-clustering of   $S$ with $d$ under  kernel $k$-means.     
\end{twierdzenie}
\begin{proof}
Assume to the contrary that not $\Gamma$ but $\Gamma'$ different from it is the optimal $k$-clustering of $S$. 
$\Gamma'$ would then contain at least one cluster $C'$ with at least two {data}points $P,R$ such that both stem from distinct clusters of $\Gamma$ (that is $|C'|=n'\ge 2$.
Hence their distance amounts to at least $\sqrt{2}\sqrt{Q(\Gamma,d)}$. 
All the other $n'-2$ elements of $C'$ fall into three categories: 
belonging under $\Gamma$ to the same cluster as $P$ ($n_P$ elements) 
belonging under $\Gamma$ to the same cluster as $R$ ($n_R$ elements) and the remaining ones ($n_s$ elements). 
$n_P+n_R+n_s= n'-2$. 
So, within the cluster $C'$ there are at least $(n_P+1)\cdot (n_R+1)+n_s\cdot (n_R+n_P+2)$ pairs of {data}points with distance  at least  $\sqrt{2}\sqrt{Q(\Gamma,d)}$. 
So the contribution of $C'$ to the quality function amounts to 
\begin{equation} 
Q(\{C'\} ,d)    
=\frac{1}{2n'} \sum_{ i \in C'} 
\sum_{ l \in C'} d(i,l) ^2  
\end{equation}   
$$\ge \frac{1}{n'}  
\left((n_P+1)\cdot (n_R+1)+n_s\cdot (n_R+n_P+2)\right) \cdot 2 Q(\Gamma,d)
$$ $$
\ge \frac{1}{n'}  
(n'-1) \cdot 2 Q(\Gamma,d)
$$
As $ Q(\{C'\} ,d) \ge Q(\Gamma,d)$,
so we have $  Q(\Gamma',d)\ge  Q(\Gamma,d)$ as claimed in this theorem. 
   $\Gamma$ is in fact optimal.     
\end{proof}

As the theorem holds for pseudo-distance, it holds also for the Euclidean distance. 

\begin{twierdzenie}
If the clustering 
$\Gamma$ of $S$ under   the pseudo-distance $d$ fulfills the condition (\ref{eq:vardist}), 
then there exists no other $\Gamma'$ of $S$ that fulfills the condition (\ref{eq:vardist}).  
\end{twierdzenie}
\begin{proof}
    As already shown in the previous theorem  \ref{thm:varksepISoptimal}, $\Gamma$ is optimal. So $\Gamma'$ would have to be optimal but different from $\Gamma$.   But this is impossible as  putting two elements from distinct clusters would significantly increase the quality function value, as seen in the previous theorem proof. 
\end{proof}


\begin{definicja}
We say that a clustering function $f(S,d)$ returns
\emph{ variation{al} $k$-clustering} of $S$ 
if $S$ is variation{ally} $k$-separable under $d$ and $f(S,d)$ returns the $\Gamma$ clustering  being 
  variation{al} $k$-separation  of $S$. 
\end{definicja}

\begin{twierdzenie} \label{thm:ax:vark}
    The variation{al} $k$-clustering $\Gamma$ will remain the variation{al} $k$-clustering  after consistency transform. In other words  consistency transform preserves clustering by a function detecting variation{al} $k$-clustering. 
\end{twierdzenie}
\begin{proof}
The increase of inter-cluster distances does not violate variation{al} $k$-separation because the distances between clusters will be larger than prescribed by the variation{al} minimal distance from formula (\ref{eq:vardist}). 
The decrease of intra-cluster separation does not violate variation{al} $k$-separation because  the variation{al} minimal distance will be smaller so distances between clusters will fit better this minimal distance.  
\end{proof}

Consider the Euclidean distances only for a moment. 
Let us ask the question how difficult it would be to discover the optimal clustering. 
Let us consider the $k$-means$++$ algorithm \cite{CLU:AV07}, or more precisely the derivation of the initial clustering. Recall that wide gaps between clusters guarantee that after hitting each cluster during the initialization stage, the optimum clustering is achieved. 
Let us consider a step when $i$ seeds have hit $i$ distinct clusters. 
Then the probability of hitting an unhit cluster in the next step amounts to:
$$\frac{SSD_{unhit}}{SSD_{unhit}+SSD_{hit}}=1- \frac{SSD_{hit}}{SSD_{unhit}+SSD_{hit}}$$
where 
$SSD_{hit}$ is the sum of squared distances to closest seed from elements of hit clusters $C_1,C_2,\dots, C_i$, and 
$SD_{unhit}$ is the sum of squared distances to closest seed from elements of unhit clusters $C_{i+1},\dots,C_k$. 
Let $C$ be a hit cluster. Then $Q(\{C\},d)$ will be the upper bound for the squared distance between any element of $C$ and the cluster center. Hence   $2Q(\{C\},d)$ will be the upper bound of the sums of squared distances between a seed from $C$ and its other elements. 
Therefore
$$SSD_{hit}\le  2 Q(\Gamma_{hit},d)\le  2 Q(\Gamma,d)$$
where $\Gamma_{hit}$ is the set of clusters hit so far. 
On the other hand
$$SSD_{unhit} \ge 2 Q(\Gamma,d) \sum_{j=i+1}^k n_j $$
where $n_j=|C_j|$. Hence 
$$ 
\frac{SSD_{unhit}}{SSD_{unhit}+SSD_{hit}} 
=
\frac{1}{1+\frac{SSD_{hit}}{SSD_{unhit}}}
$$ $$
\ge
\frac{1}{1+\frac{2 Q(\Gamma,d)}{2 Q(\Gamma,d) \sum_{j=i+1}^k n_j }} 
=
\frac{1}{1+\frac{1}{ \sum_{j=i+1}^k n_j }} 
$$ $$
=
\frac{\sum_{j=i+1}^k n_j }{\sum_{j=i+1}^k n_j +1} 
=1- 
\frac{1 }{\sum_{j=i+1}^k n_j +1} 
$$ 
If we assume that  the cardinality of all clusters is the same and equals $m$, then we have 
$$=1- \frac{1}{m(k-i)+1}
 $$
So that the overall expected probability of hitting all clusters during initialization amounts to at least 
\begin{equation}\label{eq:hitprob}
    \prod_{i=1}^{k-1} \left(1- \frac{1}{m(k-i)+1} \right)
\end{equation}

If $m$ exceeds $k$, then this probability is very close to one (assuming $m>50$). If not all clusters are of the same cardinality, but $m$ is its lower bound, then the above formula gives the lower bound on this probability.

\begin{twierdzenie}
     There exists a function  detecting variation{al} $k$-clustering with high probability  that has the property of scale-invariance, consistency and $k$-richness, given that the function operates in Euclidean space and the consistency transformation is performed in Euclidean space too. 
\end{twierdzenie}
\begin{proof}
    We have just shown that $k$-means++ can be used to detect, with high probability, variation{al} $k$-clustering, if the data lies in the Euclidean space. 
    It is known to have the property of scale-invariance. $k$-richness is easily shown: formulate a $k$-clustering $\Gamma$,  set distances between points within each cluster to values such that each cluster fits the Euclidean space, and then move the clusters in the Euclidean space in such a way that the condition (\ref{eq:vardist}) is matched and complete the distance definition. 
    The consistency property holds because of Theorem \ref{thm:ax:vark}. 
\end{proof}

We return to considering pseudo-distances. 
Let us go beyond the $k$-richness, expanding our considerations towards the concept of richness. Already Kleinberg showed that full richness does not make sense and restricted himself to near-richness. Below we restrict the concept of near-richness to range-$k_x$-richness. 
\begin{definicja}
A clustering function $f$
has the range-$k_x$-richness property if for any 
  dataset $S$ for each $\Gamma \in 2^S$ consisting of non-empty subsets of at least two elements such that $|\Gamma|\le k_x$ there exists a distance function $d$ such that $f(S,d)=\Gamma$. 
\end{definicja}

Note that the near richness imposes the restriction $|\Gamma|\le |S|-1$. It allos also for clusters with one element only which we forbid in range-$k_x$-richness. 

\begin{definicja}
We say that a clustering function $f(S,d)$ returns
\emph{ variation{al} range-$k_x$ clustering} of $S$ 
if $S$ is variation{ally} $k$-separable under $d$ for some $1\le k\le k_x $ and for  $\Gamma=f(S,d)$  
for no cluster $C\in \Gamma$
there exists  $k'$, $2\le k`\le k_x-k+1$ that  $C$ is 
variation{ally} $k'$-separable. 
The maximal $k$ with this property shall be called the level of variation{al} range-$k_x$ clustering.
\end{definicja}

Obviously, $k$-means++ would be a suitable sub-algorithm for the algorithm 
of the discovery of variation{al} range-$k_x$ clustering of a dataset, represented by the   master Algorithm~\ref{alg:var-clustering} $f()$: try out all k=$k_x$ to 2 if there exist variation{al} $k$-clustering; and if so, then check each sub-cluster on no variation{al} $k'$ separability.

\begin{algorithm}
\KwData{$S$ - a set of objects embedded in Euclidean space\\
$k_x$ - the maximal number of clusters to be obtained  }
\KwResult{ $k$ - the number of detected clusters (if 1, no clusters were detected)\\
$\Gamma$ - the clustering of $S$ into $k$ clusters }
\If{$k_x<2$}{\Return $k=1$, $\Gamma=\{S\}$}
\For{$k\leftarrow k_x$ \KwTo $2$ \KwBy $-1$} 
{ Cluster $S$ using $k$-means$++$\ getting $\Gamma$\;
  \If{$\Gamma$ ensures that according to Def.\ref{def:varsep} $S$ is variation{ally} $k$-separable}{ 
     $OK$=TRUE\;
     \For{$S'\in \Gamma$} 
     {Apply this algorithm to $S'$ with $k_x'=k_x-k+1$ obtaining $k'$ and $\Gamma'$\;
     \If{$k'\ge 2$}{$OK$=FALSE}
     }
     \If{$OK$}{\Return $k$, $\Gamma$}
                   } 
} 
\Return $k=1$, $\Gamma=\{S\}$
\caption{
The algorithm 
of the discovery of variation{al} range-$k_x$ clustering of a dataset 
}\label{alg:var-clustering}
\end{algorithm}

What will happen when performing Kleinberg's consistency operation?%
\footnote{Kleinberg's consistency operation leads outside of Euclidean space in general, but let us restrict our considerations to the case of Euclidean space.}
A cluster that is not variation{ally} $k'$ separable, may turn to a variation{ally} $k'$ separable one if we apply a consistency transformation. 
Therefore we need to restrict consistency transformation.
We suggest to replace it with the relative consistency transformation, defined as follows:
\begin{definicja}
Consider a dataset $S$ and a distance function $d: S\times S\rightarrow \mathbb{R}$ and a clustering function $f()$. Let $f(S,d)=\Gamma$.  
  Define a different distance function $d'$ such that for any cluster $C\in\Gamma$:
  (1) for  $i,j,l\in C$, 
  $d'(i,j)\le d(i,j)$ and 
  if $d(i,j)\le d(i,l)$ then $d'(i,j)\le d'(i,l)$ and $\frac{d'(i,l)}{d'(i,j)}\le \frac{d(i,l)}{d(i,j)}$,
  (2) for  $i\in C$ and $l\not\in C$
  $d'(i,l)\ge d(i,l)$.  This transformation from $d$ to $d'$ shall be called 
  \emph{relative consistency transformation}. 
\end{definicja}

The relative consistency transformation defined above differs from the consistency transformation of Kleinerg in the following way: (1) it preserves the ordering of distances within a cluster, (2) it prevents the emergence of densier areas within a cluster. In this way, no new clusters emerge within a cluster after this transformation, contrary to Kleinberg's definition. This new definition removes a crucial deficiency of Kleinerg's axiiomatic system. 

\begin{definicja}
  If  the clustering function $f$ for each data set $S$ and each distance function $d$ and each of its relative consistency transforms $d'$  has the property that $f(S,d)=f(S,d')$, then we shall say that $f$ has the property of \emph{relative consistency}. 
  \end{definicja}

\begin{twierdzenie} \label{thm:ax:var}
    The variation{al} range-$k_x$ clustering at the level $k$ will remain the variation{al} range-$k_x$ clustering at the level $k$ after relative consistency transform. In other words  relative consistency transform preserves clustering by a function detecting variation{al} range-$k_x$ clustering. 
\end{twierdzenie}
\begin{proof}
The increase of inter-cluster distances does not violate variation{al} $k$-separation because the distances between clusters will be larger than prescribed by the variation{al} minimal distance. 
The decrease of intra-cluster separation does not violate variation{al} $k$-separation because  the variation{al} minimal distance will be smaller so distances between clusters will fit better this minimal distance.  
Furthermore, the decrease of intra-cluster separation does not 
turn a non-variation{ally} separable set into a separable set for the following reason: assume S1 and S2  are two subclusters of a cluster S which we consider as candidates for being variation{ally} separated after the transformation. This implies that the distances between elements of S1 and S2 were larger than within S1 and within S2 after the operation, and so were they before the operation. But if they were larger before the operation then they are more strongly shortened than those within S1 and S2. But this means that the decrease of the variation{al} minimal distance is smaller than the decrease in distances between S1 and S2. Hence the variation{al} separation cannot occur. 
\end{proof}

This implies the following theorem.
\begin{twierdzenie}
     The clustering function described by Algorithm~\ref{alg:var-clustering}, detecting variation{al} range-$k_x$ clustering with high probability, has the property of scale-invariance, relative consistency and range-$k_x$ richness, if operating in Eucliean space. 
\end{twierdzenie}

However,
we will have a problem with the  relative consistency transformation of a distance $d$ to a distance $d'$. In general case, even if $d$ is an Euclidean distance, $d'$ does not need to be an Euclidean distance. 
As shown in \cite{RAKMAKSTW:2020:trick}, a distance function $d'$ being non-euclidean can be turned into Euclidean one $d"$ by adding an appropriate constant $\delta^2$ to each squared distance $d'(i,j)^2$ of distinct elements and the clustering with (kernel) $k$-means will preserve the $k$-clustering of $S$. However, it is possible that the property of variation{al} $k$ separability will be lost via such an adding operation. Our goal, yet, is to find the class of {data}sets and clustering functions fitting axioms that operate in the Euclidean space.

\section{Residual Cluster Separation} \label{sec:residualsep}

Assume that $\sigma(d)$ is the lowest distance $d$ over the set $S$. 
Then
\begin{twierdzenie}
Let $\Gamma$ be a clustering of the set $S$, and let  $n=|S|$. Then    
    $$Q(\Gamma,d) \ge (n-k) \frac{\sigma(d) ^2}{2}  $$
\end{twierdzenie}
\begin{proof}   
$$
Q(\Gamma,d)
\ge 
\sum_{C \in \Gamma} \frac{1}{2|C|} \sum_{i \in C} 
\sum_{l \in C; l\ne i}
 \sigma(d) ^2  
$$ $$
= \sum_{C \in \Gamma}  \frac{|C|-1}{2} \sigma(d) ^2  
= (|S|-k) \frac{\sigma(d) ^2}{2} 
$$
\end{proof}
Define the function
\begin{equation}\label{eq:defbeta}
     \beta(\Gamma,d)=2\left(Q(\Gamma,d)-(n-k-1) \frac{\sigma(d) ^2}{2}\right)
\end{equation}

Let us introduce our next concept of well-{separated}ness. 
\begin{definicja} \label{def:ressep}
Let us consider a set of clusters separated as follows: 
Let $\Gamma=\{C_1,\dots,C_k\}$ be a partition of the dataset $S$, $d$ be a pseudo-distance. 
Let   
\begin{equation}\label{eq:resdist}
d(i,l)>\sqrt{\beta(\Gamma,d) }
\end{equation}
for each $i,l$ 
such that $i$ belongs to a different cluster than $l$ under $\Gamma$.  
Then we say that the set $S$ with distance $d$ is 
\emph{ residual{ly} $k$-separable},
and $\Gamma$ is the \emph{residual $k$-separation  of $S$}. 
If furthermore, no cluster of $\Gamma$ has the property of residual $k'$-separation for all $k'=2,\dots,K+1$, then  
 $\Gamma$ is 
\emph{residual $k+K$-range-separation} of $S$.
\end{definicja}

\begin{twierdzenie}
Assume 
that the set $S$ with distance $d$ is 
\emph{ residual{ly} $k$-separable}.   Then 
$\Gamma$ minimizes $Q(\Gamma,d)$ over all clusterings of the dataset $S$. 
\end{twierdzenie}

\begin{proof}
    In analogy to the proof of the Theorem \ref{thm:varksepISoptimal}, we can demonstrate that 
    $  \beta(\Gamma',d)\ge  \beta(\Gamma,d)$. 
     Hence
    $$
     2\left(Q(\Gamma',d)-(n-k-1) \frac{\sigma(d) ^2}{2}\right)
     $$ $$
\ge 
2\left(Q(\Gamma,d)-(n-k-1) \frac{\sigma(d) ^2}{2}\right)
$$
That is 
$  Q(\Gamma',d)\ge  Q(\Gamma,d)$.
\end{proof}

\begin{twierdzenie}
Assume we have two pseudo-distance functions $d_1,d_2$ over $S$ such that for any two distinct $x,y$: $d_2^2(x,y)=d_1^2(x,y)+\Delta$ for some constant $\Delta$. Then 
$$\beta(\Gamma,d_2)= \beta(\Gamma,d_1)  + \Delta
$$     
\end{twierdzenie}
\begin{proof}
    
$$ Q(\Gamma,d_2)   
=\sum_{C \in \Gamma} \frac{1}{2|C|} \sum_{i \in C} 
\sum_{l \in C}  d_2(i,l) ^2  
$$ $$ 
=\sum_{C \in \Gamma} \frac{1}{2|C|} \sum_{i \in C} 
\sum_{l \in C; l\ne i}  \left(d_1(i,l) ^2  +\Delta\right)
$$ $$ 
=Q(\Gamma,d_1) +\sum_{C \in \Gamma} \frac{1}{2|C|} \sum_{i \in C} 
\sum_{l \in C; l\ne i}   \Delta 
=Q(\Gamma,d_1) +(n-k) \frac{\Delta}{2} 
$$   
Furthermore
$$\beta(\Gamma,d_2)=
2\left(Q(\Gamma,d_2)-(n-k-1) \frac{\sigma(d_2) ^2}{2}\right) 
$$ 
$$= 
2Q(\Gamma,d_2)
-(n-k-1)  \sigma(d_2) ^2
$$ 
$$=
2Q(\Gamma,d_1) +(n-k) \Delta
-(n-k-1)  \sigma(d_1) ^2
-(n-k-1) \Delta
$$ 
$$= 
2Q(\Gamma,d_1)  
-(n-k-1)  \sigma(d_1) ^2
+ \Delta
=
\beta(\Gamma,d_1)  
+ \Delta
$$ 
\end{proof}

The above theorem implies:
\begin{twierdzenie} \label{thm:addingDelta}
Assume we have two pseudo-distance functions $d_1,d_2$ over $S$ such that for any two distinct $x,y$: $d_2^2(x,y)=d_1^2(x,y)+\Delta$ for some constant $\Delta$. 
Then  
  the set $S$ with pseudo-distance $d_1$ is 
\emph{ residual{ly} $k$-separable} iff   the set $S$ with pseudo-distance $d_2$ is 
\emph{ residual{ly} $k$-separable}. 

\end{twierdzenie}
\begin{proof}
   As $\beta(\Gamma,d_2)=\beta(\Gamma,d_1)  + \Delta$, 
then it would be sufficient for \emph{ residual $k$-separability} of $S$ under $d_2$ that the squared pseudo-distance between elements of distinct clusters is increased by $\Delta$ which is the case by definition of $d_2$. 
So increase of distances from $d_1$ to $d_2$ preserves the residual $k$-separation.
On the other hand, if $\Gamma$ with $|\Gamma|=k$ is not a residual $k$-separation under $d_1$, then there exist two elements $i,l$ from distinct clusters such that $d_1(i,l)^2\le \beta(\Gamma,d_1)$. Therefore $d_2(i,l)^2=d_1(i,l)^2+\Delta\le \beta(\Gamma,d_1)+\Delta=\beta(\Gamma,d_2)$. 
\end{proof}

\begin{definicja}
We say that a clustering function $f(S,d)$ returns
\emph{ residual $k$-clustering} of $S$ 
if $S$ is residual{ly} $k$-separable under $d$ and $f(S,d)$ returns the $\Gamma$ clustering  being 
  residual $k$-separation  of $S$. 
\end{definicja}

\begin{twierdzenie} \label{thm:ax:resk}
    The residual $k$-clustering $\Gamma$ will remain the residual $k$-clustering  after consistency transform, given that no pseudo-distance gets shorter than the shortest distance at the beginning (\emph{lower-bounded consistency}). In other words  consistency transform preserves clustering by a function detecting variation{al} $k$-clustering. 
\end{twierdzenie}
\begin{proof}
The increase of inter-cluster distances does not violate residual $k$-separation because the distances between clusters will be larger than prescribed by the residual minimal distance from formula (\ref{eq:resdist}). 
The decrease of intra-cluster separation does not violate residual $k$-separation because  the residual minimal distance (\ref{eq:resdist}) will decrease  so distances between clusters will fit better this minimal distance.  
\end{proof}

Obviously, $k$-means++ is no more suitable for discovering residual $k$-clustering.  
We need to create the following modification of $k$-means++: (res-$k$-means++). Instead of taking squared distances to the closest seed, use the difference between it and the squared smallest distance whatsoever during initialization stage. 
Let us concentrate for a moment on Euclidean distances. 

Let us ask the question how difficult it would be to discover the optimal clustering. 
Let us consider the res-$k$-means$++$ algorithm, or more precisely the derivation of the initial clustering, whereby $\Gamma$ is the true clustering. Let us consider a step when $i$ seeds have hit $i$ distinct true clusters $\mathcal{H}$. 
For a hit cluster $C$ let $h(C)$ be the hit seed of this cluster.
Then the probability of hitting an unhit cluster  in the next step amounts to:
$$\frac{SSDM_{unhit}}{SSDM_{unhit}+SSDM_{hit}}$$
where 
$SSDN_{unhit}$ is the sum of squared distances to closest seed minus $\delta(d)^2$ from elements of unhit clusters, and   $SSDM_{hit}$ is the sum of squared distances  minus $\delta(d)^2$ to closest seed  from elements of hit clusters. 

$SSDM_{hit}=\sum_{C\in\mathcal{H}} \sum_{e\in C; e \ne h(C)} \left(d^2(e,h(C))-\sigma(d)^2\right)$.
But for any $l\in C$
$\sum_{e\in C; e\ne l} (d^2(e,j)-\sigma(d)^2)
\le 2Q(\{C\},d)-(|C|-1)\sigma(d)^2)$.
So $SSDM_{hit}\le \sum_{C\in\mathcal{H}} 
\left(2Q(\{C\},d)-(|C|-1)\sigma(d)^2)\right)\le 2Q(\Gamma,d)-(n-k)\sigma(d)^2) <\beta(\Gamma,d)
$.

On the other hand 
$SSDM_{unhit} \ge \beta(\Gamma,d) \sum_{C \in \Gamma-\mathcal{H}} |C|=
\beta(\Gamma,d) \sum_{j=1}^k n_j
$. 
Hence 
$$ 
\frac{SSDM_{unhit}}{SSDM_{unhit}+SSDM_{hit}} 
\ge
\frac{1}{1+\frac{ \beta(\Gamma,d)}{ \beta(\Gamma,d) \sum_{j=i+1}^k n_j }} 
$$ $$
=
\frac{1}{1+\frac{1}{ \sum_{j=i+1}^k n_j }} 
=
\frac{\sum_{j=i+1}^k n_j }{\sum_{j=i+1}^k n_j +1} 
=1- 
\frac{1 }{\sum_{j=i+1}^k n_j +1} 
$$ 
If we assume that  the cardinality of all clusters is the same and equals $m$, then we have 
$$=1- \frac{1}{m(k-i)+1}
 $$

So that the overall expected probability of hitting all clusters during initialization amounts to at least (as in eq. (\ref{eq:hitprob}))
$$ \prod_{i=1}^{k-1} \left(1- \frac{1}{m(k-i)+1} \right)
$$
Remarks on high probability and unequal cluster sizes are here the same as with  equation (\ref{eq:hitprob}).

\begin{twierdzenie}
     There exists a function  detecting residual $k$-clustering with high probability  that has the property of scale-invariance, lower-bounded consistency and $k$ richness, given that the function operates in Euclidean space and the consistency transformation is performed in Euclidean space too. 
\end{twierdzenie}
\begin{proof}
    We have just shown that res-$k$-means++ can be used to detect, with high probability, residual $k$-clustering, if the data lies in the Euclidean space. 
    Obviously it has the property of scale-invariance. $k$-richness is easily shown: formulate a $k$-clustering $\Gamma$,  set distances between points within each cluster to values such that each cluster fits the Euclidean space, and then move the clusters in the Euclidean space in such a way that the condition (\ref{eq:resdist}) is matched and complete the distance definition. 
    The consistency property holds because of Theorem \ref{thm:ax:resk}. 
\end{proof}

The concept of  lower-bounded consistency may appear somehow awkward from the mathematical point of view, but it is not so if we look at technical reality. Any distance measurement is restricted by some resolution factor of the measuring device. So if two points are too close they may be indistinguishable. So assuming a minimal distance between distinct data points makes technically sense.

Let us return to pseudo-distances. 

\begin{definicja}
We say that a clustering function $f(S,d)$ returns
\emph{ variation{al} range-$k_x$ clustering} of $S$ 
if $S$ is variation{ally} $k$-separable under $d$ for some $1\le k\le k_x $ and for  $\Gamma=f(S,d)$  
for no cluster $C\in \Gamma$
there exists  $k'$, $2\le k`\le k_x-k+1$ that  $C$ is 
variation{ally} $k'$-separable. 
The maximal $k$ with this property shall be called the level of variation{al} range-$k_x$ clustering.
\end{definicja}

\begin{definicja}
We say that a clustering function $f(S,d)$ returns
\emph{ residual range-$k_x$ clustering} of $S$ 
if $S$ is  residual{ly}  $k$-separable under $d$ for some $1\le k\le k_x $ and for  $\Gamma=f(S,d)$  
and
for no cluster $C\in \Gamma$
there exists  $k'$, $2\le k`\le k_x-k+1$ that  $C$ is 
residual{ly} $k'$-separable. 
The maximal $k$ with this property shall be called the level of residual range-$k_x$ clustering.
\end{definicja}

Obviously, to discover with high probability \emph{ residual range-$k_x$ clustering} in the Euclidean domain, we need to use res-$k$-means++ as sub-algorithm for  the master Algorithm~\ref{alg:res-clustering}: try out all k= $k_x$ to 2 if there exist residual $k$-clustering, and if so, then check each sub-cluster on no residual $k'$ separability.

\begin{algorithm}
\KwData{$S$ - a set of objects embedded in Euclidean space\\
$k_x$ - the maximal number of clusters to be obtained  }
\KwResult{ $k$ - the number of detected clusters (if 1, no clusters were detected)\\
$\Gamma$ - the clustering of $S$ into $k$ clusters }
\If{$k_x<2$}{\Return $k=1$, $\Gamma=\{S\}$}
\For{$k\leftarrow k_x$ \KwTo $2$ \KwBy $-1$} 
{ Cluster $S$ using res-$k$-means$++$\ getting $\Gamma$\;
  \If{$\Gamma$ ensures that according to Def.\ref{def:ressep} $S$ is residual{ly} $k$-separable}{ 
     $OK$=TRUE\;
     \For{$S'\in \Gamma$} 
     {Apply this algorithm to $S'$ with $k_x'=k_x-k+1$ obtaining $k'$ and $\Gamma'$\;
     \If{$k'\ge 2$}{$OK$=FALSE}
     }
     \If{$OK$}{\Return $k$, $\Gamma$}
                   } 
} 
\Return $k=1$, $\Gamma=\{S\}$
\caption{Residual clustering algorithm}\label{alg:res-clustering}
\end{algorithm}

Let us ask, what will happen when performing Kleinberg's consistency operation. 
The first problem that we encounter is that consistency transform performed on a cluster that is not {residual}ly $k'$ separable, may turn to {residual}ly $k'$ separable one. 
Therefore we need to restrict consistency transformation to a relative one. 
But this is not sufficient. As we make use of the concept of the smallest distance in our formulas, we need to add the restriction that the lowest distance will not be decreased.

\begin{twierdzenie}\label{thm:ax:res}
    The residual range-$k_x$ clustering at the level $k$ will remain the residual range-$k_x$ clustering at the level $k$ after lower bounded relative  consistency transform. In other words  lower-bounded relative  consistency transform  preserves clustering by a function detecting residual range-$k_x$ clustering. 
\end{twierdzenie}
\begin{proof}
The increase of inter-cluster distances does not violate residual $k$-separation. 
The decrease of intra-cluster separation according to the imposed limitations does not turn a non-{residual}ly separable set into a separable set. The proof is analogous to that of Theorem~\ref{thm:ax:var}. 
\end{proof}

This implies 
\begin{twierdzenie}\label{thm:euclidianaxioms}
     The Algorithm~\ref{alg:res-clustering}   detecting residual range-$k_x$ clustering with high probability, has the property of scale-invariance, lower-bounded relative consistency and range-$k_x$ richness in the Euclidean domain. 
\end{twierdzenie}

\section{From Euclidean space to Kleinberg's concept of distance}\label{sec:non-eucliden}

We have demonstrated that $k$-means algorithm does not need to be in conflict with Kleinberg's consistency axiom if the dataset contains clearly separated clusters. 
What is more, after a slight adjustment of consistency axiom to some real world conditions (resolution of data is finite), then based on $k$-means, a clustering algorithm can be constructed matching in practice the three Kleinberg's axioms. 
There is, however, one deficiency in the approach: it assumes that the data is embedded in Euclidean space, while Kleinberg insisted that his axioms should hold also outside the Euclidean realm.  Most of the proofs presented do not depend on Euclidean embedding. The weak point in going beyond it is the $k$-means algorithm which was designed for Euclidean space. While the development of $k$-mans++ in \cite{CLU:AV07} not bound to Euclidean space, yet guarantees of minimum of $k$-means quality function rely on the concrete properties in the Euclidean space. So-called kernel $k$-means provably seeks the same minimum as traditional k-means after Euclidization proposed by  Lingoes \cite{Lingoes:1971}, but the problem is that after this Euclidization the variational $k$-separability may be lost so that there is no guarantee that k-means seeks to optimize by finding {variational}ly $k$-separation. 

We have discussed so far the case of clustering axioms   for Euclidean distance, see Theorem~\ref{thm:euclidianaxioms}. This axiomatic system does not approximate quite what Kleinberg proposed because he used a more relaxed version of distance function, the pseudo-distance.   

So   consider the  lower-bounded relative consistency transformation applied to a pseudo-distance $d$ yielding another pseudo-distance $d'$, that is one outside  of the framework of Euclidean space.  
The proof of Theorem~\ref{thm:euclidianaxioms} can be easily converted to the case of pseudo-distances, using insights from Theorem~\ref{thm:addingDelta}. 
We need only to adapt accordingly the algorithm res-$k$-means. 
Adaptation of algorithms can follow the results from \cite{RAKMAKSTW:2020:trick}. 
As shown in \cite{RAKMAKSTW:2020:trick}, a distance function $d$ being non-euclidean can be turned into Euclidean one $d_E$ by adding an appropriate constant $\delta^2$ to each squared distance $d(i,j)^2$, 
and the clustering with  $k$-means under $d_E$  will preserve the $k$-clustering obtained via kernel $k$-means with the original distance $d$.
The Theorem \ref{thm:addingDelta} strengthens that result saying that  after residual $k$-separation property before and after this transformation is the same. 
Hence we can use the mentioned Algorithm~\ref{alg:res-clustering} for discovery of residual range-$k_x$ clustering, after transforming to Euclidean distance in the spirit of  Theorem \ref{thm:addingDelta}. 

\section{Conclusions} \label{sec:conclusions}

This research has shown that one should not throw away Kleinberg's axioms because of their contradiction. 
We have pointed at what was missing in Kleinberg's axiomatic system - that is the idea that clustering transformation functions make sense only if they are applied to a  clustering performed on a {cluster}able dataset. 
We have shown that if the dataset is {cluster}able according to a properly defined separation criterion, then $k$-means stops to be inconsistent in terms of Kleinberg, and a version of $k$-means can be created that matches all   three clustering axioms. 
It is also easily seen that single-link algorithms, used in Kleinerg's paper \cite{Kleinberg:2002} can be upgraded to match all  three Kleinberg's axioms with {cluster}able data. 

{Acknowledge}ably, the gaps between clusters used in this paper are (very) large\footnote{As they are also in other works on clusterability, e.g. by  Ostrovsky \cite{Ostrovsky:2013}} and therefore further research should seek to lower inter-cluster distances while still keeping the axiomatic system intact. 
Alternatively one may investigate the degrees of violation of Kleinberg's axiomatic system given extent to which the clusterability criteria are violated. 
 
\bibliographystyle{plain}

\begin{thebibliography}{10}

\bibitem{Ackerman:2016}
M.~Ackerman, Andreas Adolfsson, and Naomi Brownstein.
\newblock An effective and efficient approach for clusterability evaluation.
\newblock {\em CoRR}, abs/1602.06687, 2016.
\newblock earlier notions of clusterability 32,11,12,8,2,9 . 7,6.

\bibitem{Ackerman:2010NIPS}
M.~Ackerman, S.~Ben-David, and D.~Loker.
\newblock Towards property-based classification of clustering paradigms.
\newblock In {\em Advances in Neural Information Processing Systems 23}, pages
  10--18. Curran Associates, Inc., 2010.

\bibitem{CLU:AV07}
D.~Arthur and S.~Vassilvitskii.
\newblock $k$-means++: the advantages of careful seeding.
\newblock In N.~Bansal, K.~Pruhs, and C.~Stein, editors, {\em Proc. of the
  Eighteenth Annual ACM-SIAM Symposium on Discrete Algorithms}, SODA 2007,
  pages 1027--1035, New Orleans, Louisiana, USA, 7-9 Jan. 2007. SIAM.

\bibitem{Ben-David:2015}
S.~Ben-David.
\newblock Computational feasibility of clustering under clusterability
  assumptions.
\newblock \url{https://arxiv.org/abs/1501.00437}, 2015.

\bibitem{Ben-David:2009}
S.~Ben-David and M.~Ackerman.
\newblock Measures of clustering quality: A working set of axioms for
  clustering.
\newblock In D.~Koller, D.~Schuurmans, Y.~Bengio, and L.~Bottou, editors, {\em
  Advances in Neural Information Processing Systems 21}, pages 121--128. Curran
  Associates, Inc., 2009.

\bibitem{Cohen:2018}
Vincent Cohen-Addad, Varun Kanade, and Frederik Mallmann-Trenn.
\newblock Clustering redemption - beyond the impossibility of kleinberg’s
  axioms.
\newblock In S.~Bengio, H.~Wallach, H.~Larochelle, K.~Grauman, N.~Cesa-Bianchi,
  and R.~Garnett, editors, {\em Advances in Neural Information Processing
  Systems}, volume~31. Curran Associates, Inc., 2018.

\bibitem{Ding:2009}
Chris Ding.
\newblock Dimension reduction techniques for clustering.
\newblock In L.~Liu and M.T. Oezsu, editors, {\em Encyclopedia of Database
  Systems}. Springer, Boston, MA., 2009.

\bibitem{Gao:2017}
Zhi{-}Qiang Gao and Long{-}Jun Zhang.
\newblock {DPHKMS:} an efficient hybrid clustering preserving differential
  privacy in spark.
\newblock In Leonard Barolli, Mingwu Zhang, and Xu~An Wang, editors, {\em
  Advances in Internetworking, Data {\&} Web Technologies, The 5th
  International Conference on Emerging Internetworking, Data {\&} Web
  Technologies, EIDWT-2017, Wuhan, China, June 10-11, 2017}, volume~6 of {\em
  Lecture Notes on Data Engineering and Communications Technologies}, pages
  367--377. Springer, 2017.

\bibitem{Girolami:2002}
M.~Girolami.
\newblock Mercer kernel-based clustering in feature space.
\newblock {\em IEEE Transactions on Neural Networks}, 13(3):780--784, 2002.

\bibitem{Hopcroft:2012}
J.~Hopcroft and R.~Kannan.
\newblock Computer science theory for the information age, 2012.
\newblock chapter 8.13.2. A Satisfiable Set of Axioms. page 272ff.

\bibitem{Howland:2007}
P.~Howland and H.~Park.
\newblock Cluster preserving dimension reduction methods for document
  classification.
\newblock In M.W. Berry and M.~Castellanos, editors, {\em Survey of Text
  Mining:Clustering, Classiﬁcation, and Retrieval, Second Edition}, pages
  3--23. Springer, 2007.

\bibitem{Keller:2021}
Hannah Keller, Helen Möllering, Thomas Schneider, and Hossein Yalame.
\newblock Privacy-preserving clustering.
\newblock In Stefan-Lukas Gazdag, Daniel Loebenberger, and Michael Nüsken,
  editors, {\em crypto day matters 32}, Bonn, 2021. Gesellschaft f{\"u}r
  Informatik e.V. / FG KRYPTO.

\bibitem{Kleinberg:2002}
J.~Kleinberg.
\newblock An impossibility theorem for clustering.
\newblock In {\em Proc. NIPS 2002}, pages 446--453, 2002.
\newblock http://books.nips.cc/papers/files/nips15/LT17.pdf.

\bibitem{MAK:2017:clusterability}
M.~A. K{\l}opotek.
\newblock An aposteriorical clusterability criterion for k-means++ and
  simplicity of clustering.
\newblock {\em {SN} Comput. Sci.}, 1(2):80, 2020.

\bibitem{MAK:2022:kmeanspreserving}
M.~A. K{\l}opotek.
\newblock A clustering preserving transformation for k-means algorithm output.
\newblock {\em CoRR}, abs/2202.10455, 2022.

\bibitem{MAKRAK:2020:fixdimcons}
M.~A. K{\l}opotek and R.~A. K{\l}opotek.
\newblock Clustering algorithm consistency in fixed dimensional spaces.
\newblock In {\em Foundations of Intelligent Systems}, volume 12117 of {\em
  LNCS}, pages 352--361. Springer, 2020.

\bibitem{MAKRAK:2020:limcons}
M.~A. K{\l}opotek and R.~A. K{\l}opotek.
\newblock In-the-limit clustering axioms.
\newblock In {\em Artificial Intelligence and Soft Computing}, volume 12416 of
  {\em LNCS}, pages 199--209. Springer, 2020.

\bibitem{MAKSTWRAK:2020:motion}
M.~A. K{\l}opotek, S.~T. Wierzchon, and R.~A. K{\l}opotek.
\newblock k-means cluster shape implications.
\newblock In {\em Artificial Intelligence Applications and Innovations}, volume
  583 of {\em {IFIP} Advances in Information and Communication Technology},
  pages 107--118. Springer, 2020.

\bibitem{Klopotek:2022continuous}
Mieczys{\l}aw~A. K{\l}opotek and Robert~A. K{\l}opotek.
\newblock Towards continuous consistency axiom.
\newblock {\em Applied Intelligence}, 2022, 2022.

\bibitem{RAKMAKSTW:2020:trick}
R.~K{\l}opotek, M.~K{\l}opotek, and S~Wierzcho\'{n}.
\newblock A feasible k-means kernel trick under non-euclidean feature space.
\newblock {\em International Journal of Applied Mathematics and Computer
  Science}, 30(4):703--715, 2020.
\newblock Online publication date: 1-Dec-2020.

\bibitem{RAKMAK:2019:probrich}
R.~A. K{\l}opotek and M.~A. K{\l}opotek.
\newblock On probabilistic k-richness of the k-means algorithms.
\newblock In {\em Machine Learning, Optimization, and Data Science}, volume
  11943 of {\em LNCS}, pages 259--271. Springer, 2019.

\bibitem{ISMIS:2022:richness}
M.~K{\l}oppotek and R.~K{\l}opotek.
\newblock Richness fallacy.
\newblock In {\em Proc. ISMIS 2022, 03-05.10.2022, Cosenza, Italy. Lecture
  Notes in Computer Science book series (LNAI,volume 13515) Foundations of
  Intelligent Systems}, pages 262--271. Springer, 2022.

\bibitem{RAK:MAK:2019:perfectball}
Robert~A. Kłopotek and Mieczysław~A. Kłopotek.
\newblock Solving inconsistencies of the perfect clustering concept.
\newblock In {\em Proc. of PP-RAI’2019 Congress}, pages 273--276, 20129.

\bibitem{Larsen:2016heavy}
Kasper~Green Larsen, Jelani Nelson, Huy~L. Nguyundefinedn, and Mikkel Thorup.
\newblock Heavy hitters via cluster-preserving clustering.
\newblock {\em Commun. ACM}, 62(8):95–100, jul 2019.

\bibitem{Lingoes:1971}
J.C. Lingoes.
\newblock Some boundary conditions for a monotone analysis of symmetric
  matrices.
\newblock {\em Psychometrika}, 36:195--203, 1971.

\bibitem{Madhulatha:2012}
T.~Madhulatha.
\newblock An overview on clustering methods.
\newblock {\em IOSR Journal of Engineering}, 2(4):719--725, Apr. 2012.

\bibitem{Meila:2005}
M.~Meil\v{a}.
\newblock Comparing clusterings: An axiomatic view.
\newblock In {\em Proceedings of the 22Nd International Conference on Machine
  Learning}, ICML '05, pages 577--584, New York, NY, USA, 2005. ACM.

\bibitem{Ostrovsky:2013}
R.~Ostrovsky, Y.~Rabani, L.~J. Schulman, and C.~Swamy.
\newblock The effectiveness of lloyd-type methods for the k-means problem.
\newblock {\em J. ACM}, 59(6):28:1--28:22, January 2013.
\newblock 0.0000001 is epsilon so that epsilon square <target kmeans for
  k/target kmeans for k-1.

\bibitem{Parameswaran:2005}
Rupa Parameswaran and Douglas~M. Blough.
\newblock A robust data-obfuscation approach for privacy preservation of
  clustered data.
\newblock In {\em Proceedings of the Workshop on Privacy and Security Aspects
  of Data Mining}, page 18–25, 2005.

\bibitem{Ramirez:2020privacy}
Daniel~Hurtado Ramírez and J.~M. Auñón.
\newblock Privacy preserving k-means clustering: A secure multi-party
  computation approach.
\newblock arXiv 2009.10453, 2020.

\bibitem{Roth:2003}
V.~Roth, J.~Laub, M.~Kawanabe, and J.M. Buhmann.
\newblock Optimal cluster preserving embedding of nonmetric proximity data.
\newblock {\em IEEE Transactions on Pattern Analysis and Machine Intelligence},
  25(12):1540--1551, 2003.

\bibitem{Strazzeri:2021}
Fabio Strazzeri and Rub{\'e}n~J. S{\'a}nchez-Garc{\'i}a.
\newblock Possibility results for graph clustering: A novel consistency axiom.
\newblock https://arxiv.org/abs/https://arxiv.org/abs/1806.06142, 2021.

\bibitem{vanLaarhoven:2014}
T.~van Laarhoven and E.~Marchiori.
\newblock Axioms for graph clustering quality functions.
\newblock {\em Journal of Machine Learning Research}, 15:193--215, 2014.

\bibitem{STWMAK:2018:clustering}
S.T. Wierzcho\'{n} and M.A. K{\l}opotek.
\newblock {\em Modern Clustering Algorithms}.
\newblock Springer Verlag Series: Studies in Big Data 34. Springer-Verlag,
  2018.

\bibitem{Zhang:2019}
Jianpeng Zhang, Kaijie Zhu, Yulong Pei, George Fletcher, and Mykola
  Pechenizkiy.
\newblock Cluster-preserving sampling from fully-dynamic streaming graphs.
\newblock {\em Information Sciences}, 482:279--300, May 2019.

\bibitem{Zhao:2020}
Yaliang Zhao, Samwel~K. Tarus, Laurence~T. Yang, Jiayu Sun, Yunfei Ge, and
  Jinke Wang.
\newblock Privacy-preserving clustering for big data in cyber-physical-social
  systems: Survey and perspectives.
\newblock {\em Information Sciences}, 515:132--155, 2020.

\end{thebibliography}


\end{document}